\documentclass[runningheads]{llncs}
\usepackage[utf8]{inputenc}
\usepackage[T1]{fontenc}
\usepackage{amsmath}
\usepackage{amssymb}
\usepackage{mathtools}
\usepackage{booktabs}
\usepackage{array}
\usepackage{hyperref}
\usepackage{cleveref}
\usepackage{algorithm}
\usepackage{algorithmic}
\usepackage{graphicx}
\usepackage{xcolor}
\usepackage{xspace}
\usepackage{subcaption} 
\usepackage{mdframed}
\usepackage{siunitx}
\usepackage{orcidlink}

\newcommand{\rnd}[1]{\operatorname{rnd}\!\left[#1\right]}
\newcommand{\bz}{\mathbf{z}}

\newcommand{\bmm}{\mathbf{m}}
\newcommand{\bx}{\mathbf{x}}

\newcommand{\Real}{\mathbb{R}}
\newcommand{\Integer}{\mathbb{Z}}
\def\prob{{\mathbf{Pr}}}
\newcommand{\E}{\mathbb{E}}

\newcommand{\jacob}[1]{\textcolor{magenta}{\textbf{Jacob:} #1}}

\definecolor{RoyalBlue}{RGB}{65,105,225}

\title{Integer Natural Evolution Strategies} %:\\ with Individual Step-Size}
\author{Jacob de Nobel\inst{1}\orcidlink{0000-0003-1169-1962},
Diederick Vermetten\inst{2}\orcidlink{0000-0003-3040-7162},
Hao Wang\inst{1}\orcidlink{0000-0002-4933-5181},\\
Ofer M.\ Shir\inst{3}\orcidlink{0000-0002-8257-5160}, 
Michael Emmerich\inst{4}\orcidlink{0000-0002-7342-2090},
Thomas B\"ack\inst{1}\orcidlink{0000-0001-6768-1478}
}
\institute{Leiden University, Leiden, The Netherlands
\and 
Sorbonne Universit\'e, CNRS, LIP6, Paris, France
\and
Tel-Hai - University of Kiryat Shmona in the Galilee, Israel
\and
Faculty of Information Technology, University of Jyv{\"a}skyl{\"a}, Finland
}
\authorrunning{De Nobel et al.}

\begin{document}
\maketitle

\begin{abstract}
While contemporary Evolution Strategies handle integer optimization problems effectively, their adaptation mechanism is grounded in $\ell_2$-based Gaussian models, which are not native to the integer lattice. In contrast, the $\ell_1$-norm provides the natural measure of displacement on $\mathbb{Z}^n$, with the double geometric distribution as its canonical mutation operator. In this work, we derive a fully $\ell_1$-native step-size adaptation mechanism from first principles and propose an Integer Natural Evolution Strategy. We show that the DG distribution belongs to the exponential family, and that its sufficient statistic $|z|$ yields a natural-gradient signal for dispersion adaptation. By accumulating this signal via an evolution path, we obtain a fading-memory online estimator of the natural gradient, following Ollivier (2018). This establishes that DG-based step-size adaptation arises directly from the statistical structure of the mutation distribution, rather than as a discrete analog of continuous ES mechanisms. Empirical results on integer quadratic benchmarks show that \textsc{INES} learns meaningful coordinate-wise step-sizes and is competitive with integer-handling CMA-ES baselines. Its advantages are most visible in high-dimensional Ellipsoidal problems and in robust convergence at larger dimensions.
\end{abstract}

\section{Introduction}\label{sec:introduction}
Integer Optimization focuses on formulating and solving models whose objective and constraint functions involve variables restricted to the integer lattice $\mathbb{Z}^n$. 
Research in this area is closely allied with Mixed-Integer (MI) Optimization (for a recent review see \cite{Talbi2024}), primarily because the vast majority of real-world integer problems arise in MI forms, requiring a blend of real- and discrete-valued variables.
The discrete domain induces a different geometry and fundamentally alters the search-space landscape; specifically, the $\ell_1$ norm captures the inherent grid-like connectivity more faithfully than the $\ell_2$ norm, and discretization can transform unimodal continuous objectives into multimodal landscapes riddled with local traps \cite{shir2025}.
Within the landscape of heuristic optimization, Evolution Strategies (ESs)~\cite{Beyer-Schwefel} represent a highly effective class of randomized search heuristics. They are distinguished from other Evolutionary Algorithms by their rigorous grounding in both biological principles and theoretical analysis and by their sophisticated self-adaptation of search parameters \cite{Beyer-Schwefel}. 
In addition, ESs are invariant under monotonic transformations of the objective function and angle-preserving transformations of the decision space, which includes rotations and translations~\cite{Beyer-Schwefel}. While predominantly applied to continuous domains, the ES framework has been extended to treat integer and mixed spaces \cite{RudolphHandbookNACO,Hamano2025cat}. 
To date, however, these adaptations largely rely on truncated variants of mutation distributions originally conceived for continuous optimization \cite{marty2024lb}.
In this work, we move beyond these truncated approximations to investigate the effectiveness of the \emph{Double Geometric} (DG) distribution~\cite{rudolph1994}. %in unbounded integer search spaces. 
As the discrete counterpart to the continuous Laplace distribution \cite{Inusah2006}, the DG has long been recognized as a natural choice for such domains. Recently, \cite{shir2025} reaffirmed its advantages over the commonly used \emph{Truncated Normal} (TN) and demonstrated the feasibility of capturing correlations between integer variables. 

In the continuous domain, however, research has progressed from classical self-adaptation to \emph{derandomized} parameter control, which aims to update strategy parameters without any independent stochastic variation. 
This culminated in the Covariance Matrix Adaptation Evolution Strategy (CMA-ES)~\cite{hansen2001completely}, which is the current state of the art for continuous black-box optimization. Despite its success, analogous derandomized mechanisms have not been systematically developed for (mixed) integer optimization. This gap stems from the difficulty of modeling discrete correlations and the lack of a natural analog of the covariance matrix on the integer lattice. 
Consequently, and somewhat controversially, one of the most competitive strategies for integer optimization remains the direct application of CMA-ES with only minimal modifications~\cite{marty2024lb}, despite its fundamental design for continuous manifolds.

Another reason to move from Euclidean to lattice-native step-size control is that a fixed $\ell_2$ radius does not determine the effective mutation length after applying the componentwise rounding map $\rnd{\cdot}$ to $\Integer^n$. 
Indeed, points on the same Euclidean shell may induce vastly different discrete $\ell_1$ lengths after rounding, depending on how their mass is distributed across coordinates. Thus, Euclidean truncation alone is not a reliable proxy for the mutation strength (See ~\cite{zenodo}). 
%Appendix~\ref{app:l1l2-discrepancy} 
% We precisely quantify the resulting bandwidth in an online Supp.\ Material.%\footnote{\td{Zenodo URL for Michael's Appendix.}}

% \medskip

% \noindent
Against this background, we address the following fundamental question:

\noindent
\textbf{Research Question.} Can step-size adaptation in integer ES be derived from first principles using an $\ell_1$-native mutation model and information-geometric principles, and does this yield competitive performance on $\mathbb{Z}^n$?
\medskip

\noindent \textbf{Contributions.}
We provide a principled approach to step-size adaptation in Integer Evolution Strategies (IESs):
\begin{itemize}
    \item \textbf{Natural-gradient derivation:} We show that the double geometric (DG) distribution forms an exponential family whose sufficient statistic $|z|$ yields a natural-gradient signal for step-size adaptation, providing a first-principles justification for $\ell_1$-native updates.
    
    \item \textbf{Integer Natural Evolution Strategy:} We develop a $(1,\lambda)$ integer natural evolution strategy that combines DG-based mutation with a fading-memory natural-gradient estimator, resulting in a simple and stable $\ell_1$-native adaptation mechanism.
    
    \item \textbf{Empirical evaluation:} We demonstrate on integer quadratic benchmarks that the proposed method learns meaningful coordinate-wise step-sizes and delivers competitive, robust performance against both full- and separable-integer CMA-ES baselines.    
\end{itemize}

\noindent \textbf{Paper Organization.}
The remainder of this paper is organized as follows. Section~\ref{sec:background} introduces the DG distribution and IES fundamentals. Section~\ref{sec:approach} presents the $\ell_1$-native natural gradient and the resulting IES. Section~\ref{sec:results} provides the empirical evaluation, and Section~\ref{sec:conclusion} concludes.
\section{Preliminaries} \label{sec:background}
\textbf{Evolution Strategies} (ESs) are the predominant heuristics for continuous optimization, with derandomized variants, most notably CMA-ES~\cite{hansen2001completely}, representing the state-of-the-art. In ES, a current decision vector $\bx^{(t)}$ is perturbed by a mutation vector $\bz$ to generate a candidate solution
\[
    \bx^{(t+1)} = \bx^{(t)} + \bz,
\]
where $\bz$ is drawn from a parameterized search distribution. Depending on the model, the components $z_i$ may be treated as independent or correlated. A defining feature of ES is the \emph{adaptation of the mutation distribution itself}. The parameters of the search distribution, such as the step size $\sigma$ for the multivariate normal distribution, are updated over time based on the observed fitness values, allowing the algorithm to adapt its search behavior to the underlying landscape. 

% \subsection{Natural Evolution Strategies}
\textbf{Natural Evolution Strategies} (NESs)~\cite{Wierstra2014NES} take a principled distributional view of this adaptation: rather than adapting strategy parameters heuristically, NESs explicitly maintain a parameterized search distribution $\pi(\bz \mid \theta)$ and minimize the expected cost
\[
J(\theta) := \mathbb{E}_{\theta}\!\left[f(\bx+\bz)\right]
\]
by following the \textit{search gradient}, which reads
\begin{align}
    \nabla_{\theta} J(\theta)
    = \mathbb{E}_{\theta}\!\left[
        f(\bx+\bz)\,\nabla_{\theta}\log \pi(\bz \mid \theta)
      \right],
\end{align}
and can be statistically approximated via a population of mutations $\{\bz_k\}_{k=1}^{\lambda}$ as
\[
\nabla_{\theta} J(\theta)
\approx
\frac{1}{\lambda}\sum_{k=1}^{\lambda}
f(\bx+\bz_k)\,\nabla_{\theta}\log \pi(\bz_k \mid \theta).
\]
Grounded in Information Geometry~\cite{Amari1998}, NESs precondition this gradient by the inverse Fisher information matrix $\mathcal{F}(\theta)^{-1}$, yielding the \textit{natural gradient}
\[
\widetilde{\nabla}_{\theta} J(\theta)
:=
\mathcal{F}(\theta)^{-1}\nabla_{\theta} J(\theta),
\]
which is invariant under reparameterizations of $\pi$.
The resulting update rule,
\begin{align}
    \theta \leftarrow \theta - \eta\, \mathcal{F}(\theta)^{-1} \nabla_{\theta} J(\theta),
\end{align}
with $\eta$ serving as a learning rate, steers the search distribution toward regions of lower objective value in a geometrically principled manner~\cite{Wierstra2014NES}.

\subsection{Going Discrete: The Double Geometric Distribution}
Extending ESs to Integer Optimization requires rethinking both the mutation distribution and step-size control. While Gaussian distributions and $\ell_2$-based measures are natural in continuous spaces, they do not align with the geometry of the integer lattice $\Integer^n$. Instead, $\ell_1$-based quantities and lattice-native mutation operators provide a more appropriate foundation~\cite{shir2025}. This choice is supported by information-theoretic considerations: Rudolph~\cite{rudolph1994} showed that, under a constraint on the expected absolute step-size, the DG distribution is the maximum-entropy distribution on~$\Integer$. 
The magnitude of mutations on the integer lattice is naturally quantified by the expected $\ell_1$-norm
\[
S := \mathbb{E}\!\left[\|\bz\|_1\right]
   = \sum_{i=1}^{n}\mathbb{E}\!\left[|Z_i|\right],
\]
which characterizes the overall mutation magnitude in~$\Integer^n$. 
In the following, we primarily work with the per-coordinate quantities $\delta_i := \mathbb{E}[|Z_i|]$, which serve as the fundamental step-size parameters.

\medskip
\noindent
The DG distribution can be constructed by drawing two independent geometric random variables $G_1,G_2 \sim \mathrm{Geom}(p)$ with
\[
\prob\!\left\{G_{\jmath}=k\right\} = p(1-p)^k,
\qquad k\in\mathbb{N}_0,
\]
and setting $Z = G_1 - G_2$. 
Its probability mass function (PMF) is
\begin{equation}\label{eq:probDGeometric}
    \prob\!\left\{Z=k\right\}
    =
    \frac{p}{2-p}(1-p)^{|k|},
    \qquad k\in\mathbb{Z}.
\end{equation}
It is convenient to \emph{reparameterize} using
\(
q := 1 - p, \quad
q \in (0,1)
\)
-- yielding the form $\prob\{Z=k\} \propto q^{|k|}$, which we adopt throughout the paper. \\
The DG distribution has the following key statistics:
\begin{equation}
\mathbb{E}[Z]=0,
\quad
\operatorname{Var}[Z]=\frac{2q}{(1-q)^2}.
\end{equation} 
The expected absolute step-size of a single DG variable is related to $q$ via the following invertible \emph{map}
\begin{equation}\label{eq:delta-of-q}
\delta(q)
=
\mathbb{E}[|Z|]
=
\frac{2q}{1-q^2} \quad\Longleftrightarrow\quad 
q(\delta)
=
\frac{\delta}{\sqrt{1+\delta^2}+1}~~.
\end{equation}
In the isotropic case, where all coordinates share the same parameter, the global step-size for the DG distribution satisfies $S_{\mathrm{DG}} = n \cdot \delta$. 

The $\ell_1$-based notion of step-size is not merely a modeling preference, but reflects a genuine geometric discrepancy between Euclidean and lattice-native mutation length. To see this, let $\rnd{\bx}$ denote componentwise rounding of $\bx\in\Real^n$ to the nearest integer vector, with exact halves rounded downward, and consider the rounded $\ell_1$ length
$$
L(\bx) := |\rnd{\bx}|_1
\qquad \text{subject to}\qquad
|\bx|_2 = d .
$$
Then $L(\bx)$ is not determined by the Euclidean radius $d$ alone. If the Euclidean mass is spread sufficiently thinly across coordinates, for example uniformly over more than $4d^2$ coordinates, then each component has magnitude below $1/2$, so that $\rnd{\bx}=\mathbf{0}$ and hence $L(\bx)=0$. By contrast, for the balanced direction $\bx=(d/\sqrt{n},\dots,d/\sqrt{n})$, one obtains
$$
L(\bx)=n\left\lceil \frac{d}{\sqrt{n}}-\frac12 \right\rceil .
$$
In particular, this quantity is at least $n$ whenever $n<4d^2$, and at least $2n$ whenever $n\le d^2/4$.
Thus, vectors lying on the same $\ell_2$ shell can correspond, after rounding, to a large range of different effective $\ell_1$ mutation lengths. This is one reason why the expected absolute step-size $S=\E!\left[|\bz|_1\right]$ is the natural quantity to control on $\Integer^n$, whereas Euclidean truncation does not by itself characterize the mutation strength seen by the lattice.
A detailed derivation is provided as Supp.\ Material~\cite{zenodo}.

\subsection{Perspective: Exponential Families and Information Geometry}\label{sec:ExpFamily}
We now situate the DG distribution within the framework of exponential families and Information Geometry. This perspective provides a direct link between distribution parameters and the step-size quantities $\delta_i = \mathbb{E}[|Z_i|]$ introduced earlier. Modeling discrete search updates through this lens has strong precedent in Estimation of Distribution Algorithms (EDAs), where Information Geometry provides a principled basis for probability model updates~\cite{Malago2011}. Under this geometric view, parameter updates correspond to steepest ascent in distribution space with respect to the Fisher--Rao metric, i.e., natural gradient ascent~\cite{Amari1998}. This mirrors the continuous setting, where covariance matrix adaptation in Gaussian ESs can be interpreted as natural gradient ascent on the Gaussian ($\ell_2$) manifold~\cite{Akimoto2010}. Our framework transposes this principle to the discrete ($\ell_1$) domain.

A PMF $\pi(\bz \mid \boldsymbol{\theta})$ belongs to the exponential family if it can be written as
\begin{equation}\label{eq:exp-family-general}
    \prob(\bz \mid \boldsymbol{\theta})
    =
    h(\bz)\exp\!\bigl(
        \boldsymbol{\theta}^{\!\top}\boldsymbol{\phi}(\bz) - A(\boldsymbol{\theta})
    \bigr),
\end{equation}
where $\boldsymbol{\theta}$ denotes the natural parameters, $\boldsymbol{\phi}(\bz)$ the sufficient statistic, $h(\bz)$ the base measure, and $A(\boldsymbol{\theta})$ the log-partition function.\\
The DG distribution~\eqref{eq:probDGeometric} admits this form coordinate-wise as
\begin{equation}\label{eq:dg-pure-exponential}
    \prob(z_i \mid q_i)
    =
    \frac{1-q_i}{1+q_i}
    \exp\!\Bigl(|z_i|\ln q_i\Bigr),
\end{equation}
which identifies the natural parameter $\theta_i = \ln q_i$, sufficient statistic $\phi(z_i)=|z_i|$, base measure $h(z_i)=1$, and log-partition function
\[
A(\theta_i)
=
\ln\!\left(\frac{1+q_i}{1-q_i}\right)
=
\ln\!\left(\frac{1+e^{\theta_i}}{1-e^{\theta_i}}\right).
\]
Using the inversion map \eqref{eq:delta-of-q}, the exponential-family structure yields the identities
\begin{equation}\label{eq:A-moments}
    A'(\theta_i) = \delta_i,
    \qquad
    A''(\theta_i) = \operatorname{Var}\!\left[|z_i|\right]
    = \delta_i\sqrt{1+\delta_i^2},
\end{equation}
establishing that the natural parameters directly control the coordinate-wise step-sizes, while $A''(\theta_i)$ provides the Fisher information governing the scaling of natural gradient updates.
%\footnote{By standard exponential-family theory, $A'(\theta)=\E_\theta\!\left[|Z|\right]$~\cite{BarndorffNielsen2014}. The explicit differentiation for the DG distribution is omitted due to space limitations.}

\subsection{Score Function and Maximum Likelihood Estimation}
Since $\phi(z)=|z|$ is the sufficient statistic of the DG family (Section~\ref{sec:ExpFamily}), the absolute deviation provides the natural signal for step-size adaptation on the integer lattice. For a single observation $z$, the log-likelihood reads
\[
\mathcal{L}(\theta;\,z) = \phi(z)\,\theta - A(\theta).
\]
By differentiating with respect to $\theta$, and by using the exponential-family identity $A'(\theta)=\delta=\E_\theta[\phi(Z)]$ \cite{BarndorffNielsen2014}, one obtains the \emph{score function}
\begin{equation}\label{eq:score}
\boxed{
\nabla_{\!\theta}\mathcal{L}
=
\phi(z)-A'(\theta)
=
|z|-\delta.
}
\end{equation}

\paragraph{Maximum likelihood estimation.}
For $n$ i.i.d.\ observations $z^{(1)}, \ldots, z^{(n)}$ drawn from a common $\mathrm{DG}(p)$, the total log-likelihood is
\[
\mathcal{L}(\theta; z^{(1)}, \ldots, z^{(n)})
=
\theta \sum_{j=1}^{n}\phi\bigl(z^{(j)}\bigr) - nA(\theta).
\]
By the Fisher--Neyman factorization theorem, the sum $\sum_{j=1}^n \phi\bigl(z^{(j)}\bigr)$ is the sufficient statistic for $\theta$.
The maximum-likelihood estimator $\hat{\theta}$ must satisfy the \textbf{moment matching} condition~\cite{Murphy2012ML}:
\[
A'(\hat{\theta})
=
\frac{1}{n}\sum_{j=1}^{n}\phi\bigl(z^{(j)}\bigr)
\quad \iff \quad
\mathbb{E}_{\hat{\theta}}\!\left[\phi(Z)\right] = \bar{\phi},
\]
where $\bar{\phi}$ is the empirical mean of the sufficient statistic.\\
Crucially, no information beyond $\bar{\phi}$ is required to estimate the dispersion of DG mutations. Neither the sample variance nor higher-order moments are necessary; \textbf{the first moment of the absolute deviations is a complete descriptor} for the $\ell_1$-native lattice dispersion.

\section{Integer Natural Evolution Strategy by First Principles}\label{sec:approach}
In the $(1,\lambda)$ setting, the selected mutation $\bz_s$ provides one observation per generation. For coordinate~$i$, the corresponding DG score \eqref{eq:score} is
\begin{align}
\displaystyle \nabla_{\theta_i} \mathcal{L} = \phi(z_{s,i}) - \E_{\theta_i}[\phi(Z_i)] = |z_{s,i}| - A'(\theta_i).
\end{align}
This is the score (log-likelihood gradient) of the DG dispersion parameter at that coordinate: it is positive when the observed step exceeds the model expectation, indicating that the dispersion should increase, and negative when it falls short, indicating that it should decrease. This parallels the Gaussian/$\ell_2$ setting, where the sufficient statistic is~$z^2$, and cumulative step-size adaptation~\cite{hansen2001completely} is driven by $\|\bz\|_2^2 - \E[\|\mathcal{N}(\mathbf{0}, \mathbf{I})\|_2^2]$.

\paragraph{Fisher information and the natural gradient form.}
The natural gradient $\widetilde{\nabla}_{\theta_i}$ is obtained by scaling the score $\nabla_{\theta_i} \mathcal{L}$ by the inverse Fisher information~$\mathcal{F}(\theta_i)^{-1}$. 
Since the DG distribution is a one-parameter exponential family coordinatewise, the Fisher information reduces to the variance of the sufficient statistic:
\[
  \mathcal{F}(\theta_i)
  = \mathrm{Var}\!\bigl[\phi(Z_i)\bigr]
  = A''(\theta_i).
\]
Thus, the natural gradient takes the explicit form:
\begin{align}\label{eq:DG-nat-grad}
\boxed{
\displaystyle \widetilde{\nabla}_{\theta_i}
  = \frac{\phi(z_{s,i}) - \E[\phi(Z_i)]}
         {\mathrm{Var}\!\bigl[\phi(Z_i)\bigr]}
  = \frac{|z_{s,i}| - \E[|Z_i|]}
         {\mathrm{Var}\!\bigl[|Z_i|\bigr]}\,.}
\end{align}
%%%
\paragraph{Statistical-estimation perspective.}
The innovation $|z_{s,i}| - \E[|Z_i|]$ is the score of the DG exponential family, and thus the canonical signal for estimating its dispersion parameter~\cite{BarndorffNielsen2014}. By the Fisher--Neyman factorization theorem~\cite{Murphy2012ML}, no information beyond the sample value of $|z_{s,i}|$ is required to update the model. Ollivier~\cite{Ollivier2018} showed that exponentially discounted accumulation of the natural gradient for exponential-family parameters yields an online natural-gradient estimator, equivalent to an extended Kalman filter. 
This motivates defining the evolution path $\boldsymbol{\phi}$ as a \emph{fading-memory accumulator of the natural gradient}:
\begin{align}\label{eq:evo-path}
\boxed{
\displaystyle \boldsymbol{\phi} \;\leftarrow\; (1-c)\,\boldsymbol{\phi}
  \;+\; c \cdot \widetilde{\nabla}_{\boldsymbol{\theta}} \,,}
\end{align}
where $\widetilde{\nabla}_{\boldsymbol{\theta}}$ denotes the coordinate-wise natural gradient defined in~\eqref{eq:DG-nat-grad}. The normalization by the Fisher information ensures that path entries across coordinates contribute on a common, Fisher-calibrated scale; a practical regularization is discussed in Remark~\ref{rem:var-normalizer}. 
%\hw{
\begin{remark} %Due to Hao
% \td{Hao: address the $\phi_k$-remark by Reviewer-3; possibly also refer to the fading memory idea raised by Reviewer-4?}  
We can interpret the fading-memory accumulator/evolution path in Eq.~\eqref{eq:evo-path}
through the lens of dynamical systems and show that it stabilizes the learning
of the DG parameters. Let $g(\theta) = \mathcal{F}(\theta)^{-1}\nabla_{\theta} \mathcal{L}$ 
denote the natural gradient. Consider a (local) maximum $\theta^*$ of $\mathcal{L}$, i.e., $g(\theta^*) = 0$ and 
$J \coloneqq \mathrm{d}g(\theta^*) \prec 0$. First, since the fading-memory accumulator $\phi$ estimates the $\theta^*$ in an online manner, we can understand $\phi_k$ as a momentum in iteration $k$:
$$
\theta_{k+1} = \theta_{k} + \eta \phi_{k+1}, \quad \phi_{k+1} = (1-c)\phi_k + c g(\theta_k), \quad k=0, 1,2,\ldots
$$
where $\eta > 0$ is the learning rate.
Next, define $\Delta_k \coloneqq  \theta_k - \theta^*$ and linearize the natural gradient around $\theta^*$, i.e., $g(\theta_k) = J\Delta_k$. We have the linear system $\phi_{k+1} = (1-c)\phi_k + c J \Delta_k, \Delta_{k+1} = \Delta_k + \eta \phi_{k+1}$. The stability of the system is determined by the spectrum of $J$.
Let $\rho$ be an eigenvalue of $J$ with corresponding eigenvector $v$.
Restricting the dynamics to this direction by writing
$\Delta_k = x_k v$ and $\phi_k = y_k v$, we obtain the scalar recursion
$y_{k+1} = (1-c) y_k + c\rho x_k$ and $
x_{k+1} = x_k + \eta y_{k+1}$
which can be written as
\[
\begin{pmatrix}
x_{k+1}\\
y_{k+1}
\end{pmatrix}
=
\begin{pmatrix}
1 + \eta c \rho & \eta(1-c)\\
c\rho           & 1-c
\end{pmatrix}
\begin{pmatrix}
x_k\\
y_k
\end{pmatrix}.
\]
The system is stable along $v$ if and only if the eigenvalues of this matrix
lie inside the unit disk, which yields the condition
\[
\eta < \frac{2(2-c)}{c|\rho|}.
\]
For comparison, the plain gradient system
$\Delta_{k+1} = \Delta_k + \eta J \Delta_k$
is stable under the stricter condition $\eta < 2/|\rho|$.
Since $2(2-c)/c > 2$ for all $0 < c < 1$, the accumulator enlarges the
admissible range of learning rates. In other words, it stabilizes the dynamics
and permits more aggressive updates. 
This effect is analogous to momentum methods in continuous optimization,
where temporal smoothing improves stability of the update dynamics.
\end{remark}

%%%
\paragraph{Verification}
We formalize this interpretation using the result of Ollivier~\cite{Ollivier2018}.
% %%%
\begin{proposition}\label{prop:evo-path}
Let $\theta_i = \ln q_i$ be the DG natural parameter for coordinate~$i$, and let $\mathcal{F}(\theta_i) = A''(\theta_i) = \mathrm{Var}_{\theta_i}[\phi(Z_i)]$ denote the corresponding Fisher information. The evolution path update
\[
\displaystyle  \phi_i \;\leftarrow\; (1-c)\,\phi_i
  \;+\; c \cdot \frac{\phi(z_{s,i}) - \E_{\theta_i}[\phi(Z_i)]}{\mathcal{F}(\theta_i)}
\]
is a fading-memory online natural gradient estimator for~$\theta_i$ \emph{\cite[Proposition~3]{Ollivier2018}}, with learning rate~$c$.
\end{proposition}

\begin{proof}
The DG distribution is a one-parameter exponential family with natural parameter $\theta_i$, sufficient statistic $\phi(z)=|z|$, and log-partition function~$A(\theta_i)$. 
For a single observation~$z_{s,i}$, the score is $\partial\mathcal{L}/\partial\theta_i = \phi(z_{s,i}) - A'(\theta_i) = |z_{s,i}| - \E_{\theta_i}[|Z_i|]$. 

The Fisher information for this family is the variance of the sufficient statistic, $\mathcal{F}(\theta_i) = A''(\theta_i) = \mathrm{Var}_{\theta_i}[|Z_i|]$. The natural gradient step for~$\theta_i$ is therefore $\mathcal{F}(\theta_i)^{-1}\bigl(|z_{s,i}| - \E_{\theta_i}[|Z_i|]\bigr)$. 
Accumulating this signal with exponential discounting at rate~$c$ yields the stated update for $\phi_i$, which coincides with the fading-memory online natural gradient of~\cite[Proposition~3]{Ollivier2018} applied coordinate-wise to the DG family. $\qed$
\end{proof}
\noindent
Consequently, DG-based cumulation follows directly from the exponential-family structure of the mutation model. The score provides the maximum-likelihood innovation, whose exponentially discounted accumulation constitutes a natural gradient estimator in the sense of~\cite{Ollivier2018}. This ensures that the evolution path~$\boldsymbol{\phi}$ is grounded in the Information Geometry of the underlying lattice distribution.
%%%
\paragraph{Score function for the location parameter.}
For completeness, we justify the center-of-mass update $\bmm \leftarrow \bmm + \bz_s$ from a maximum-likelihood perspective. Given an observation $x_i = m_i + z_i$, where $z_i \sim \mathrm{DG}(p_i)$, the log-likelihood as a function of the location parameter $m_i$ is $\mathcal{L}(m_i) = \theta_i |x_i - m_i| + \mathrm{const}$. The corresponding score (for $z_{s,i} \neq 0$) is
\begin{equation}\label{eq:location-score}
  \frac{\partial \mathcal{L}}{\partial m_i}
  \;=\;
  -\mathrm{sign}(z_{s,i})\,\theta_i\,.
\end{equation}
Since $\theta_i = \log q_i < 0$, this yields $|\theta_i|\,\mathrm{sign}(z_{s,i})$, i.e., a direction aligned with the selected mutation. Thus, the update $\bmm \leftarrow \bmm + \bz_s$ follows the maximum-likelihood gradient direction for the location parameter. The effective step length is governed by the dispersion parameters, which are adapted via the evolution path $\boldsymbol{\phi}$.

The DG exponential-family structure thus provides principled maximum-likelihood signals for both location and dispersion: the former through $\mathrm{sign}(z_{s,i})$, and the latter through the residual $\phi(z_{s,i}) - \E[\phi(Z_i)]$.
\subsection*{A Concrete $(1,\lambda)$-INES}\label{sec:ines} 
We now combine the preceding derivations into a single algorithmic framework.
The method maintains a center-of-mass $\bmm \in \Integer^n$, a vector of expected absolute step-sizes $\boldsymbol{\delta} \in \Real_{>0}^n$, and a cumulative path $\boldsymbol{\phi} \in \Real^n$. Initially, $\delta_i = \delta_0$ for all $i$ and $\boldsymbol{\phi} = \mathbf{0}$. The following block summarizes a complete generation of a $(1,\lambda)$-INES.

\begin{mdframed}
\setlength{\textfloatsep}{6pt}
\setlength{\intextsep}{6pt}
% \textbf{Sampling and selection.}
$\lambda$ offspring are independently DG-sampled and evaluated:
% \begin{algorithm}[H]
% \centering
% \caption{Sampling and selection}
\begin{algorithmic}[1]
\FOR{$k = 1,\dots,\lambda$}
    \STATE Sample $\bz_k \sim \bigotimes_{i=1}^n \mathrm{DG}(p_i(\delta_i))$
    \STATE $\bx_k \gets \bmm + \bz_k$, evaluate $f(\bx_k)$
\ENDFOR
\STATE $s \gets \arg\min_k f(\bx_k)$
\STATE $\bmm \gets \bx_s$
\end{algorithmic}
% \end{algorithm}
%\medskip
\noindent
% \textbf{Natural gradient signal and cumulation.}
Fisher-normalized centered sufficient statistic as the adaptation signal:
% \begin{algorithm}[H]
% \centering
% \caption{Natural gradient cumulation}
\begin{algorithmic}[1]
\STATE $g_i \gets \dfrac{|z_{s,i}| - \delta_i}{\max(\delta_i\sqrt{1+\delta_i^2},\,1)}$
\STATE $\phi_i \gets (1-c)\phi_i + \sqrt{c(2-c)}\, g_i$
\end{algorithmic}
% \end{algorithm}
\noindent
% \textbf{Parameter update.}
Parameters follow the accumulated natural gradient:
% \begin{algorithm}[H]
% \centering
% \caption{\textsc{INES}: parameter update}
\begin{algorithmic}[1]
\STATE $\delta_i \gets \delta_i\,\exp(\eta\, \phi_i)$
\end{algorithmic}
% \end{algorithm}
%
%\noindent The procedure then repeats from the sampling step.
%
\end{mdframed}

\begin{remark}[Direct step-size parametrization]\label{rem:reparam}
A direct update in the natural parameter \(\theta_i = \log q_i\) is problematic, as the domain constraint \(q_i \in (0,1)\) implies \(\theta_i < 0\). Additive updates in \(\theta\) do not preserve this constraint and may lead to invalid parameters (\(q_i \geq 1\)), at which point the DG moments diverge. To avoid explicit clipping, we instead parameterize the mutation strength directly via the expected absolute step size \(\delta_i\) and perform updates in this space. Since \(\delta_i > 0\) by definition, the update rule preserves positivity and thus always induces a valid DG parameter.
\end{remark}

\begin{remark}[Path normalization]\label{rem:path-norm}
The factor $\sqrt{c(2-c)}$ keeps the variance of the cumulative path $\boldsymbol{\phi}$ independent of the time constant $c$. Without this scaling, longer memory (smaller $c$) would shrink $\boldsymbol{\phi}$, effectively reducing the learning rate. The normalization, therefore, decouples the adaptation time scale from the update magnitude.
\end{remark}

\begin{remark}[Practical normalization]\label{rem:var-normalizer}
The normalization of $g_i$ corresponds to scaling by the Fisher information of the DG family, yielding a natural gradient update~\cite[Section~2.2]{Ollivier2018}. This removes the dependence of the update magnitude on the local curvature
of the model. In practice, when the distribution becomes highly concentrated, the variance, i.e. $\delta_i\sqrt{1+\delta_i^2}$, becomes small, leading to excessively large updates. We therefore apply a lower bound via $\max(\cdot,1)$, which acts as a Tikhonov-style regularization and prevents numerical instability.
\end{remark}

\subsubsection*{$\ell_1$-Nativeness and Geometric Separation}
It is instructive to distinguish between the two geometric layers underlying the proposed algorithm.
\paragraph{Data-space geometry.}
Mutation magnitudes on the integer lattice $\Integer^n$ are measured entirely in terms of absolute deviations. The resulting adaptation signal is based on the centered sufficient statistic $|z_i| - \E[|Z_i|]$, making the method fully $\ell_1$-native. The first moment $\delta_i = \E[|Z_i|]$ therefore acts as the natural scale parameter, and neither squared steps nor
Euclidean norms enter the search dynamics.
\paragraph{Statistical-manifold geometry.}
While the search itself is $\ell_1$-based, the adaptation signal is preconditioned by the Fisher information of the DG family. This introduces a second moment that reflects the local curvature of the statistical model and yields a natural gradient update. Crucially, this second moment acts only on the scalar quantity $|Z_i|$ and calibrates parameter updates; it does not affect how distances are measured in $\Integer^n$.

\section{Numerical Observations}\label{sec:results}
To empirically validate our proposed approach, we conduct a series of numerical procedures to calibrate and test it\footnote{Due to space limitations, the complete experiments are in our Supp. Material~\cite{zenodo}}.
We consider axis-aligned convex benchmark models induced by weighted quadratic forms on the integer domain. For each independent instance, the optimum $\bx^\star$ is sampled uniformly from $[-50,50]^n$, and all algorithms are initialized from the same starting point, sampled from the same domain. The objective is
\begin{equation}
    f(\bx) = \|\mathbf{W}(\bx-\bx^\star)\|_2,
\end{equation}
where $\mathbf{W}=\mathrm{diag}(w_1,\ldots,w_n)$ controls the conditioning. We study four benchmark classes:
\[
\begin{array}{ll}
\text{Sphere}:  & w_i = 1 \quad \forall i,\\
\text{Ellipse}: & w_i = {10^4}^{(i-1)/(n-1)},\\
\text{Discus}:  & w_1=10^4,\; w_i=1 \text{ for } i>1,\\
\text{Cigar}:   & w_1=1,\; w_i=10^4 \text{ for } i>1.
\end{array}
\]

%\footnote{\td{Zenodo repo?}}
%%
\paragraph{Parameter Calibration}
We calibrated our proposed \textsc{INES} using a \textit{grid search} over the Sphere and Ellipsoid functions across multiple problem dimensionalities. Based on this, we fit a simple scaling law to the optimal tradeoff points,  resulting in parameter settings $c(n) = 1 - \frac{1.5}{n}$ and $\eta(n) = \left(\frac{2}{n}\right)^{1 / 3}$, used throughout our experiments.
%The parameters of our proposed algorithm are calibrated by using a grid search on the Sphere and Ellipsoid functions across multiple problem dimensionalities. Based on this, we fit a simple scaling law to the optimal tradeoff points, resulting in the following parameter settings, which we will use throughout our experiments: 
% \[
% $c(n) = 1 - \frac{1.5},
% \qquad
% \eta(n) = \left(\frac{2}{n}\right)^{1 / 3}.
% \]
%%
\subsection{Analysis of Learned Step Sizes}
\begin{figure}[t]
    \centering
    \includegraphics[width=\linewidth]{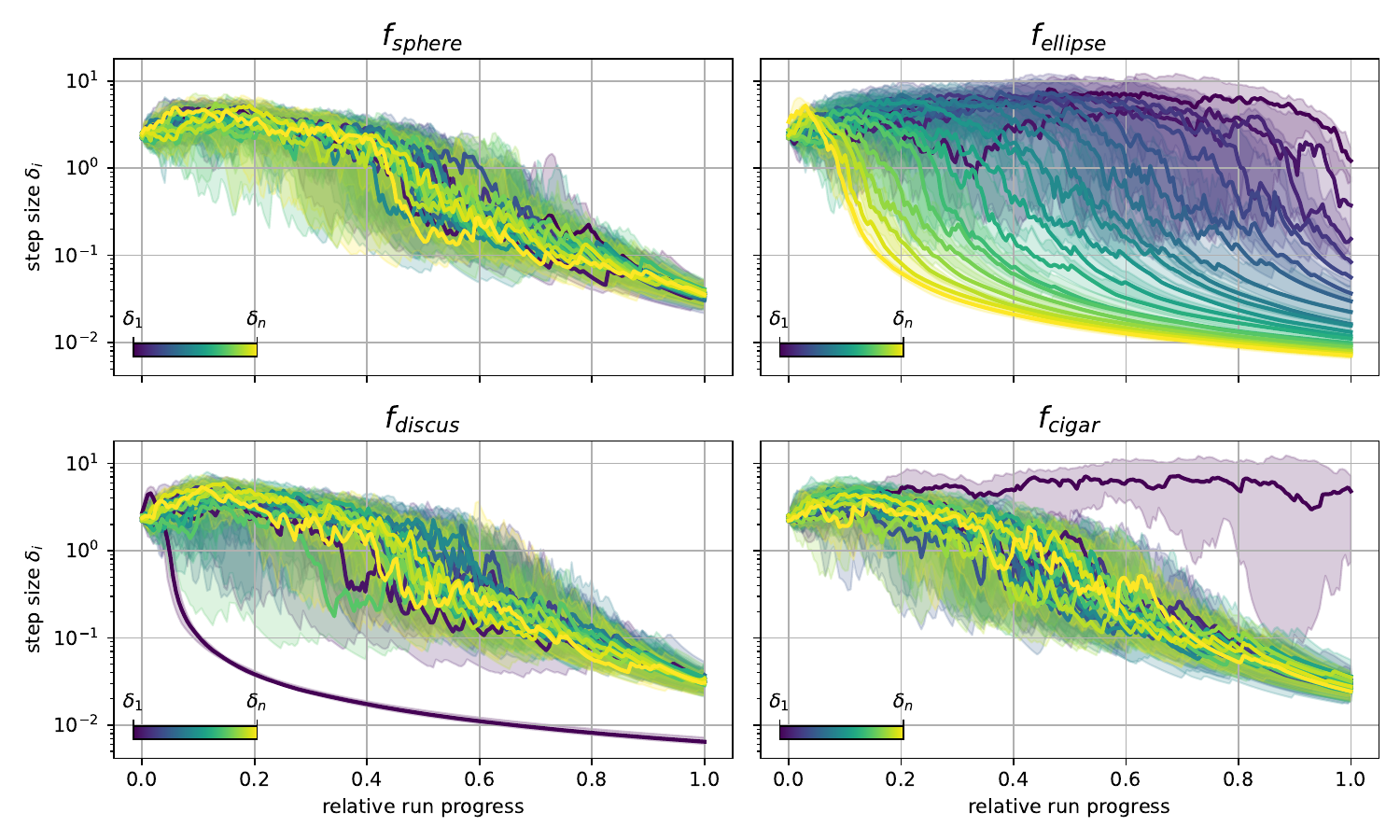}
    \caption{Evolution of coordinate-wise step-sizes $\delta_i$ over relative run progress for $n=20$. Solid lines show the median over $25$ runs, with shaded areas indicating the interquartile range. Colors correspond to coordinate indices.}
    \label{fig:vlearning10}
\end{figure}
We first focus on how the proposed method adapts coordinate-wise step-sizes. For each function and dimension $n \in \{2,3,5,10,20,40,100\}$, we perform $25$ independent runs, each terminated upon reaching the optimum. We analyze the behavior of the raw step-sizes $\delta_i(t)$ directly. To compare runs of different lengths, trajectories are resampled to a common relative time scale $\tau \in [0,1]$. %Results are aggregated using the median and interquartile range.

\paragraph{Dynamics of step-sizes.}
Figure~\ref{fig:vlearning10} shows the evolution of step-sizes for $n=20$. On the Sphere function, all coordinates evolve similarly over time, reflecting the absence of directional structure. On the Ellipse, a clear separation emerges, with step-sizes adapting to the underlying conditioning. On the Discus and Cigar functions, the method successfully identifies the distinguished direction: the corresponding coordinate exhibits a markedly different scale (suppressed for Discus, amplified for Cigar), while the remaining coordinates evolve in a largely symmetric manner. Despite the stochasticity of the updates, the qualitative behavior is consistent across runs and dimensionalities (see \cite{zenodo}) and altogether validates our \textsc{INES} design.

As an additional qualitative check, we inspect the learned step-sizes on the pseudo-Boolean benchmarks \textsc{OneMax} and \textsc{LeadingOnes}. To apply DG mutations on the binary domain, integer mutations are cyclically mapped to $\{0,1\}^n$ by $(\bx+\bz)\bmod 2$. These experiments are not intended as a full PBO benchmark study, but illustrate (see Figure~\ref{fig:vlearningpbo}) how the same adaptation mechanism behaves on problems with different coordinate structures. On \textsc{OneMax}, the learned step-sizes remain comparatively uniform, whereas on \textsc{LeadingOnes} a pronounced coordinate-dependent pattern emerges.

\begin{figure}[t]
    \centering
    \includegraphics[width=\linewidth]{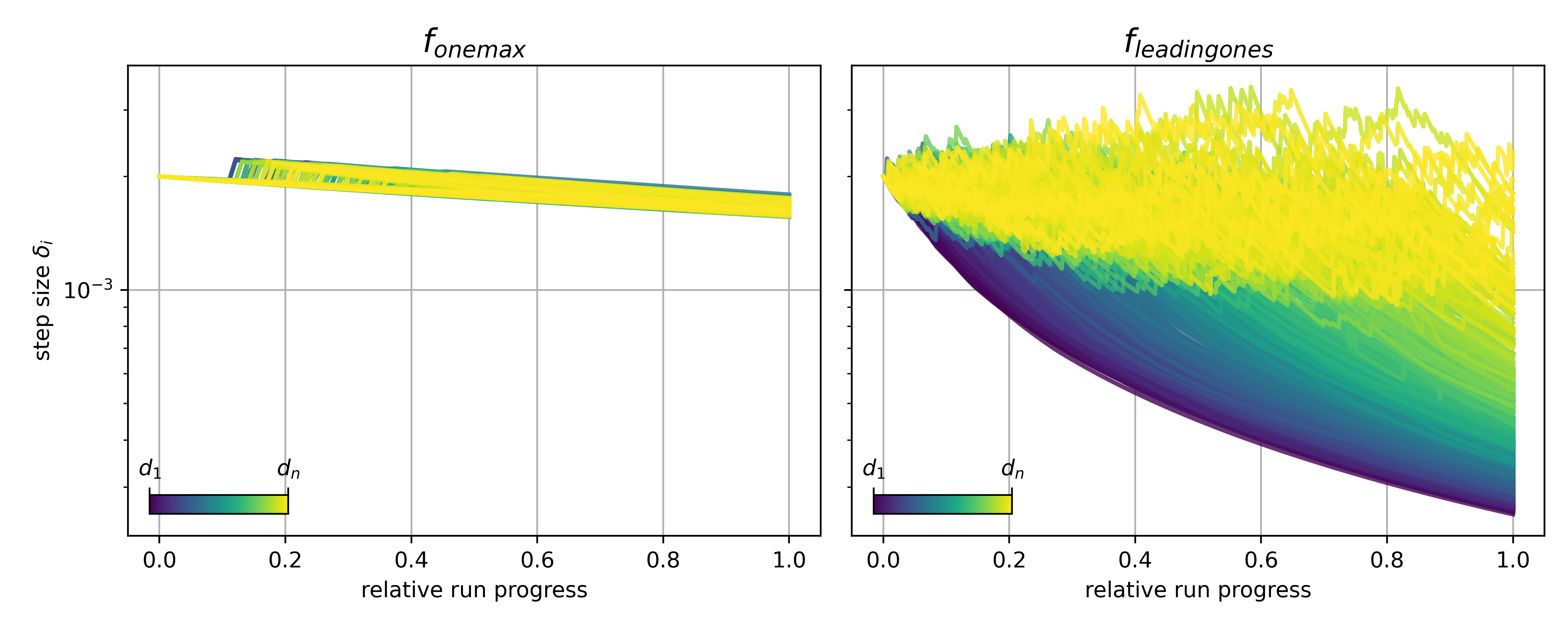}
    \caption{Evolution of coordinate-wise step-sizes $\delta_i$ on \textsc{OneMax} and \textsc{LeadingOnes} for $n=500$. Lines show the median over 25 runs; colors indicate coordinate index. The initial step-size is $\delta_i=1/n$. }
    \label{fig:vlearningpbo}
\end{figure}

\subsection{Performance on Integer Quadratic Models}
After confirming that the step size matches expected behavior, we move on to benchmark optimization performance. 
The experimental protocol follows the same setup as the previous experiment. 
As reference methods, we consider two integer-handling variants of CMA-ES. The first is \textsc{CMA-IH}~\cite{marty2024lb}, which applies TN-based mutations over the integer lattice while adapting the CMA-ES mechanism to this discrete domain. The second is its separable variant, denoted \textsc{CMA-IH-sep}, which restricts the covariance matrix to its diagonal elements~\cite{ros2008simple}. This separable baseline is particularly relevant here because the current \textsc{INES} variant also adapts only coordinate-wise step-sizes. We use the default population size for each dimensionality as recommended in the CMA-ES literature, without additional parameter tuning, and set the initial step size to $\sigma_0 = 100/n$. This choice provides a common initialization across dimensions, but the sensitivity of the CMA-based baselines to $\sigma_0$ is not studied here and remains a limitation of the empirical comparison. To obtain an initial variance of $\sigma_0^2$ for the DG distribution, we set $\delta_0 = \sigma_0^2 / {(2 \sqrt{\sigma_0^2 + 1})}$ for \textsc{INES}. \jacob{this should be: $\delta_0 = \sigma_0^2 / \sqrt{2\sigma_0^2+1}$.}

\begin{table}[t]
\centering
\caption{Expected running time (ERT) on the four benchmark functions over 25 independent runs. Lower is better; $\infty$ indicates that the target was not reached within the budget of $10^4 \cdot n$ evaluations in any run. Best values per function--dimension pair are shown in boldface. }

\label{tab:ert}

\setlength{\tabcolsep}{5pt}        % horizontal spacing
\renewcommand{\arraystretch}{1.2}  % vertical spacing

\begin{tabular}{llrrrrrrr}
\toprule
Function & Method & 2 & 3 & 5 & 10 & 20 & 40 & 100 \\
\midrule

Sphere
& \textsc{INES}   & 231 & 329 & 511 & 1,238 & 2,865 & 7,358 & \textbf{20\,904} \\
& \textsc{CMA-IH} & 115 & 217 & 423 & 933 & 2,200 & \textbf{4\,342} & $\infty$ \\
& \textsc{CMA-IH-sep} & \textbf{108} & \textbf{181} & \textbf{345} & \textbf{744} & \textbf{1\,736} & 4\,343 & 88\,395 \\ \hline

Ellipse
& \textsc{INES}   & 266 & 437 & \textbf{626} & \textbf{1\,288} & \textbf{3\,117} & \textbf{8\,402} & \textbf{29\,488} \\
& \textsc{CMA-IH} & \textbf{152} & 588 & 985 & 3\,684 & 13\,965 & 49\,498 & $\infty$ \\
& \textsc{CMA-IH-sep} & 171 & \textbf{344} & 806 & 2\,070 & 5\,246 & 16\,024 & $\infty$ \\ \hline

Discus
& \textsc{INES}   & 287 & 461 & 668 & 1\,389 & 3\,140 & 7\,927 & \textbf{23\,244} \\
& \textsc{CMA-IH} & 173 & 297 & 647 & 1\,547 & 3\,712 & 8\,232 & $\infty$ \\
& \textsc{CMA-IH-sep} & \textbf{172} & \textbf{292} & \textbf{548} & \textbf{1\,118} & \textbf{2\,496} & \textbf{4\,635} & 24\,075\,579 \\ \hline

Cigar
& \textsc{INES}   & 260 & 331 & 518 & 1\,111 & 2\,821 & 7\,184 & 22\,717 \\
& \textsc{CMA-IH} & \textbf{156} & 277 & 679 & 1\,737 & 5\,013 & 21\,351 & 667\,002 \\
& \textsc{CMA-IH-sep} & 164 & \textbf{231} & \textbf{449} & \textbf{919} & \textbf{2\,076} & \textbf{4\,193} & \textbf{13\,116} \\ \hline

\bottomrule
\end{tabular}
\end{table}

\paragraph{Results} Table~\ref{tab:ert} reports the expected running time (ERT) across
dimensionalities~\cite{hansen2012real}. All finite ERT values are based on 25 successful runs, except for \textsc{CMA-IH-sep} on Discus at $n=100$ (one success) and \textsc{CMA-IH} on Cigar at $n=100$ (22 successes); these ERTs may exceed the per-run budget because failed runs contribute their full budget. On the Sphere function, the CMA-based baselines are faster in low and moderate dimensions, but at $n=100$ \textsc{CMA-IH} fails to reach the target, and \textsc{CMA-IH-sep} requires substantially more evaluations than \textsc{INES}. \textsc{INES} remains robust in high dimensions, but the lack of a global step-size component hinders performance on the isotropic Sphere model. \textsc{INES} shows its clearest advantage on the Ellipse function, where it obtains the best ERT from $n=5$ onward and is the only method that reaches the target at $n=100$. In fact, even for a very large $n = 1000$ (not included in the Table), \textsc{INES} consistently finds the global optimum for the Ellipse model with an ERT of 399\,483. This supports the step-size analysis above, indicating that the coordinate-wise DG adaptation captures the dominant conditioning structure. On Discus and Cigar, \textsc{CMA-IH-sep} is faster on most lower- and moderate-dimensional settings, especially on Cigar, but also deteriorates strongly on Discus at $n=100$. Overall, these results show that \textsc{INES} is an effective lattice-native solver with robust high-dimensional behavior.

\section{Conclusion}\label{sec:conclusion}

We introduced the $(1,\lambda)$ Integer Natural Evolution Strategy (INES), a fully $\ell_1$-native approach to step-size adaptation on the integer lattice. The method combines double geometric (DG) mutations with a natural-gradient update derived from the exponential-family structure of the DG distribution. In this setting, the sufficient statistic $\phi(z)=|z|$ yields the adaptation signal $|z_i|-\E[|Z_i|]$, whose Fisher-scaled accumulation provides an online estimator of the natural gradient. For a practical algorithm, we use a fading-memory evolution path and directly parametrize the expected absolute step size $\delta_i=\E[|Z_i|]$, thereby avoiding validity issues arising from unconstrained updates of the natural parameter $\theta$. This yields a simple, numerically stable adaptation mechanism consistent with the geometry of $\mathbb{Z}^n$. 

Empirically, \textsc{INES} performs competitively on integer quadratic benchmarks and learns coordinate-wise step-sizes that reflect the underlying conditioning. Additional qualitative experiments on pseudo-Boolean benchmarks suggest that the same mechanism induces problem-dependent step-size dynamics on binary domains. Its strongest results appear for the Ellipse function and in several high-dimensional settings where the CMA-based baselines fail or perform poorly. At the same time, the current formulation is limited to coordinate-wise adaptation and lacks both a global step-size component and dependencies between variables, which likely explains its weaker performance on simple isotropic landscapes.

Overall, DG-based step-size adaptation is not an ad hoc discrete analog of continuous ES mechanisms, but follows directly from the statistical structure of the underlying mutation distribution. Future work includes convergence-rate analysis and extensions to multi-parent recombination. On the algorithmic side, the natural next steps are to integrate a global step-size component and to model dependencies between variables. Empirically, broader evaluations should include rotated, non-separable, and multimodal benchmarks, as well as comparisons with recent mixed-integer optimizers such as CMA-ES with Margin~\cite{hamano2022cma} and DX-NES-ICI~\cite{ikeda2023natural}.

\begin{credits}
\subsubsection{\ackname}
Diederick Vermetten acknowledges funding by the European Union (ERC, ``dynaBBO'', grant no.~101125586).
\subsubsection{\discintname}
 The authors have no competing interests to declare that are relevant to the content of this article.

\end{credits}

\bibliographystyle{splncs04}
\bibliography{main_clean}

@inproceedings{Akimoto2010,
  author    = {Youhei Akimoto and Yuichi Nagata and Isao Ono and Shigenobu Kobayashi},
  title     = {Bidirectional Relation between {CMA} Evolution Strategies and Natural Evolution Strategies},
  booktitle = {PPSN XI},
  year      = {2010},
  pages     = {154--163},
  publisher = {Springer Berlin Heidelberg}
}

@inproceedings{ros2008simple,
  title={A simple modification in CMA-ES achieving linear time and space complexity},
  author={Ros, Raymond and Hansen, Nikolaus},
  booktitle={International conference on parallel problem solving from nature},
  pages={296--305},
  year={2008},
  organization={Springer}
}

@inproceedings{Hamano2025cat,
author = {Hamano, Ryoki and Saito, Shota and Nomura, Masahiro and Uchida, Kento and Shirakawa, Shinichi},
title = {CatCMA: Stochastic Optimization for Mixed-Category Problems},
year = {2024},
isbn = {9798400704949},
publisher = {Association for Computing Machinery},
address = {New York, NY, USA},
doi = {10.1145/3638529.3654198},
pages = {656–664},
numpages = {9},
location = {Melbourne, VIC, Australia},
series = {GECCO '24},
booktitle = {Proceedings of the Genetic and Evolutionary Computation Conference},
}

@article{Amari1998,
  author  = {Amari, Shun-Ichi},
  title   = {Natural Gradient Works Efficiently in Learning},
  journal = {Neural Comput.},
  volume  = {10},
  number  = {2},
  pages   = {251--276},
  year    = {1998}
}

@book{BarndorffNielsen2014,
  author    = {Ole Barndorff-Nielsen},
  title     = {Information and Exponential Families In Statistical Theory},
  publisher = {John Wiley and Sons},
  year      = {2014}
}

@article{Beyer-Schwefel,
  author  = {Hans-Georg Beyer and Hans-Paul Schwefel},
  title   = {Evolution Strategies - A Comprehensive Introduction},
  journal = {Natural Computing},
  volume  = {1},
  number  = {1},
  pages   = {3--52},
  year    = {2002}
}

@misc{zenodo,
  author    = {De Nobel, Jacob and Vermetten, Diederick and Wang, Hao and Shir, Ofer M. and Emmerich, Michael and B{\"a}ck, Thomas},
  title     = {Integer Natural Evolution Strategies - Reproducibility and Additional Results},
  publisher = {Zenodo},
  year      = {2026},
  doi          = {10.5281/zenodo.19627585},
  OPTurl          = {https://doi.org/10.5281/zenodo.19627585},
}

@inproceedings{hamano2022cma,
author = {Hamano, Ryoki and Saito, Shota and Nomura, Masahiro and Shirakawa, Shinichi},
title = {{CMA-ES} with Margin: Lower-Bounding Marginal Probability for Mixed-Integer Black-Box Optimization},
booktitle = {Proceedings of the Genetic and Evolutionary Computation Conference},
series = {GECCO '22},
pages = {639--647},
year = {2022},
publisher = {Association for Computing Machinery},
address = {New York, NY, USA},
doi = {10.1145/3512290.3528827}
}

@inproceedings{ikeda2023natural,
author = {Ikeda, Koki and Ono, Isao},
title = {Natural Evolution Strategy for Mixed-Integer Black-Box Optimization},
booktitle = {Proceedings of the Genetic and Evolutionary Computation Conference},
series = {GECCO '23},
pages = {831--838},
year = {2023},
publisher = {Association for Computing Machinery},
address = {New York, NY, USA},
doi = {10.1145/3583131.3590518}
}

@article{hansen2012real,
  author  = {Hansen, Nikolaus and Auger, Anne and Finck, Steffen and Ros, Raymond},
  title   = {Real-parameter black-box optimization benchmarking: Experimental setup},
  journal = {Tech. Rep.},
  year    = {2012}
}

@article{hansen2001completely,
  author  = {Hansen, Nikolaus and Ostermeier, Andreas},
  title   = {Completely Derandomized Self-Adaptation in Evolution Strategies},
  journal = {Evol. Comput.},
  volume  = {9},
  number  = {2},
  pages   = {159--195},
  year    = {2001}
}

@article{Inusah2006,
  author  = {Seidu Inusah and Tomasz J. Kozubowski},
  title   = {A Discrete Analogue of the Laplace Distribution},
  journal = {Journal of Statistical Planning and Inference},
  volume  = {136},
  number  = {3},
  pages   = {1090--1102},
  year    = {2006}
}

@inproceedings{Malago2011,
  author    = {Malag\`{o}, Luigi and Matteucci, Matteo and Pistone, Giovanni},
  title     = {Towards the Geometry of Estimation of Distribution Algorithms Based on the Exponential Family},
  booktitle = {FOGA 11},
  pages     = {230--242},
  publisher = {Association for Computing Machinery},
  year      = {2011}
}

@inproceedings{marty2024lb,
  author    = {Marty, Tristan and Hansen, Nikolaus and Auger, Anne and Semet, Yann and H\'{e}ron, S\'{e}bastien},
  title     = {LB+IC-CMA-ES: Two Simple Modifications of CMA-ES to Handle Mixed-Integer Problems},
  booktitle = {PPSN XVIII},
  pages     = {284--299},
  publisher = {Springer-Verlag},
  year      = {2024}
}

@book{Murphy2012ML,
  author    = {Murphy, K.P.},
  title     = {Machine Learning: A Probabilistic Perspective},
  series    = {Adaptive Computation and Machine Learning series},
  publisher = {MIT Press},
  year      = {2012}
}

@article{Ollivier2018,
  author  = {Yann Ollivier},
  title   = {Online Natural Gradient as a Kalman Filter},
  journal = {Electronic Journal of Statistics},
  volume  = {12},
  number  = {2},
  pages   = {2930--2961},
  year    = {2018}
}

@inproceedings{rudolph1994,
  author    = {Rudolph, G{\"u}nter},
  editor    = {Davidor, Yuval and Schwefel, Hans-Paul and M{\"a}nner, Reinhard},
  title     = {An Evolutionary Algorithm for Integer Programming},
  booktitle = {PPSN III},
  year      = {1994},
  pages     = {139--148},
  publisher = {Springer Berlin Heidelberg}
}

@inbook{RudolphHandbookNACO,
  author    = {Rudolph, G{\"u}nter},
  title     = {Handbook of Natural Computing: Theory, Experiments, and Applications},
  chapter   = {Evolutionary Strategies},
  pages     = {673--698},
  publisher = {Springer-Verlag},
  year      = {2012}
}

@inproceedings{shir2025,
  author    = {Shir, Ofer M. and Emmerich, Michael},
  title     = {Foundations of Correlated Mutations for Integer Programming},
  booktitle = {FOGA 18},
  pages     = {285--296},
  publisher = {Association for Computing Machinery},
  series    = {FOGA'25},
  year      = {2025}
}

@article{Talbi2024,
  author  = {El-Ghazali Talbi},
  title   = {Metaheuristics for Variable-Size Mixed Optimization Problems: A Unified Taxonomy and Survey},
  journal = {Swarm and Evolutionary Computation},
  volume  = {89},
  pages   = {101642},
  year    = {2024}
}

@article{Wierstra2014NES,
  author  = {Wierstra, Daan and Schaul, Tom and Glasmachers, Tobias and Sun, Yi and Peters, Jan and Schmidhuber, J\"{u}rgen},
  title   = {Natural Evolution Strategies},
  journal = {J. Mach. Learn. Res.},
  volume  = {15},
  number  = {1},
  pages   = {949--980},
  year    = {2014}
}

\end{document}